\documentclass[letterpaper, 10 pt, conference]{ieeeconf}  
\IEEEoverridecommandlockouts                              

\usepackage{xcolor}
\usepackage{amsmath,amsfonts,amsthm,amssymb}  
\usepackage[font=footnotesize]{caption}
\usepackage[caption=false, font=footnotesize]{subfig}
\usepackage{algorithm}
\usepackage[noend]{algorithmic}
\usepackage{mathtools}
\usepackage{graphicx}         
\usepackage{svg}              

\usepackage[colorlinks=true, linkcolor=blue, citecolor=red, urlcolor=blue, bookmarks=true]{hyperref}
\usepackage{booktabs}

\usepackage[compress]{cite}
\definecolor{isparsblue}{HTML}{0072B2}

\let\labelindent\relax
\usepackage{enumitem}   
\usepackage{tikz}       
\usepackage{soul} 
\usepackage{wrapfig}
\usepackage{multirow}
\usepackage{multicol}

\allowdisplaybreaks

\usepackage{booktabs}
\usepackage{multirow}
\usepackage{adjustbox}

\usepackage{xspace}
\usepackage{cleveref}
\usepackage{bm}

\usepackage{etoolbox}
\newtheorem{theorem}{Theorem}
\newtheorem{claim}{Claim}

\theoremstyle{definition}

\theoremstyle{invariant}
\newtheorem{invariant}{Invariant}

\newcommand{\algname}[1]{\textsf{#1}\xspace}
\newcommand{\spars}{\algname{SPARS}}
\newcommand{\dense}{\algname{Dense}}
\newcommand{\iris}{\algname{IRIS}}
\newcommand{\irisc}{\algname{IRIS-C}}
\newcommand{\ispars}{\algname{Inspection-SPARS}}

\newtoggle{arxiv_version}
\toggletrue{arxiv_version}

\usepackage{float}

\newcommand{\ignore}[1]{}

\def\epsilon{\varepsilon}

\newif\ifshowcomments
\showcommentstrue          

\showcommentsfalse

\ifshowcomments
  \newcommand{\kiril}[1]{\textcolor{red}{(\textbf{Kiril:} #1)}}
  \newcommand{\adir}[1]{\textcolor{orange}{(\textbf{Adir:} #1)}}
  
  \newcommand{\todo}[1]{\textcolor{cyan}{(\textbf{TODO:} #1)}}
  \newcommand{\rewrite}[1]{\textcolor{olive}{(\textbf{Rewrite:} #1)}}
\else
  \newcommand{\kiril}[1]{\ignorespaces}
  \newcommand{\adir}[1]{\ignorespaces}
  \newcommand{\todo}[1]{\ignorespaces}
  \newcommand{\rewrite}[1]{\ignorespaces}
\fi

\def\niceparagraph#1{\vspace{5pt} \noindent \textbf{#1}}

\theoremstyle{definition}

\newif\ifincludeappendix
\includeappendixtrue  

\begin{document}

\title{Inspection-SPARS:  Task-Oriented Sparse Roadmaps for Inspection Planning}

\iftoggle{arxiv_version}{
    \author{
        Adir Morgan$^{*}$, Oren Salzman, and Kiril Solovey%
        \thanks{$^{*}$Corresponding author.}%
        \thanks{The authors are with the Technion--Israel Institute of
        Technology, Haifa, Israel.
        Emails: {\tt\small
        samorgan@campus.technion.ac.il,
        osalzman@cs.technion.ac.il,
        kirilsol@technion.ac.il}.}%
    }
}{
    \author{Anonymous Authors}
}

\maketitle

\begin{abstract}
Inspection planning seeks a minimum-length collision-free robot tour that
observes a given set of points of interest (POIs).
Sampling-based methods reduce this continuous problem to a \emph{graph
inspection planning} (GIP) problem over a discrete roadmap, which is then
solved using combinatorial solvers.
Dense roadmaps capture diverse inspection viewpoints and motion shortcuts, and
thus admit higher-quality solutions, but they induce large combinatorial search
spaces on which state-of-the-art GIP solvers struggle to find good solutions
within practical time budgets.
Roadmap sparsification---restructuring a dense roadmap into a compact
representation that preserves connectivity and path lengths---can alleviate
this burden.
However, existing sparsification approaches are either agnostic to the
underlying inspection task, or strive to ensure coverage of the POIs without
accounting for the quality of the resulting inspection plan.
We present \ispars, which is, to our knowledge, the first inspection-roadmap sparsifier with POI coverage and path-quality guarantees relative to the dense roadmap. 
To this end, we generalize the \spars framework, a popular task-agnostic sparsifier, from purely geometric criteria to task-oriented ones, introducing an inspection-aware vertex admission mechanism that treats POI coverage as a first-class sparsification criterion alongside connectivity and path quality. 
Experiments in realistic 3D environments show that \ispars reduces vertex and
edge counts by $\bm{4\text{--}8\times}$ while preserving coverage, allowing
the GIP solver to compute tours up to $\bm{25\%}$ shorter than with the dense roadmap or state-of-the-art inspection roadmap. 
More broadly, \ispars shows that sparsification can be made task-aware without sacrificing guarantees on solution quality.
\end{abstract}


\section{Introduction}
In \emph{inspection planning} (IP), a robot equipped with a sensor, such as a camera or LiDAR, is tasked with finding a minimum-length collision-free tour from which a given set of \emph{points of interest} (POIs) is observed.
Sampling-based approaches for IP discretize this continuous problem by constructing a motion-planning roadmap whose vertices represent robot configurations annotated with the POIs visible from them~\cite{fu2023asymptotically}.
Given such a discrete roadmap, the problem is reduced to the \emph{graph inspection planning} (GIP) problem: finding a minimum-length closed walk whose visited vertices collectively observe all required POIs.
Modern combinatorial GIP solvers search for such tours over the roadmap and can provide high-quality solutions together with optimality guarantees~\cite{morgan2026scalable,mizutani2024leveraging}.

Scaling GIP solvers to realistic environments requires high-resolution roadmaps that capture narrow geometric features and diverse POI viewpoints. A denser roadmap captures finer motion details, improving the best plan represented by the graph, but it also inflates the number of candidate plans a solver must sift through, hindering its ability to find good ones. This exposes a fundamental tension: the roadmap must be rich enough to represent high-quality inspection plans, yet compact enough to admit effective combinatorial search. Sparsification offers a natural resolution, retaining only the vertices and edges needed to represent such plans.

Sparsification has been widely explored in motion planning in the absence of inspection constraints. 
One approach introduced reachability-based admission rules for candidate PRM vertices that preserve free-space coverage and connectivity to reduce the roadmap size~\cite{nissoux1999visibility}. 
Graph-spanner methods~\cite{peleg1989graph} were subsequently applied to motion-planning roadmaps to reduce their edge count while approximately preserving shortest-path distances~\cite{marble2013asymptotically,wang2013fast}. 
More sophisticated sparsification methods, such as \spars~\cite{dobson2014sparse} and \algname{RSEC}~\cite{salzman2014sparsification}, reason about the structure of the continuous configuration space to remove geometrically redundant vertices while preserving connectivity and bounded path stretch.
%
Unfortunately, these methods base their decisions exclusively on geometric structure rather than inspection information, 
and could therefore eliminate vertices that provide critical POI vantage points, degrading plan quality.

Sparsification has also been considered in the context of GIP. 
Early randomized approaches sparsified possible viewpoints into a small set of inspection vertices, that are connected into a tour~\cite{danner2000randomized}. Englot and Hover~\cite{englot2012sampling} similarly separated coverage viewpoint selection from motion planning
optimization between viewpoints, while Urtasun et al.~\cite{urtasun2024sparse} constructed a sparse bipartite visibility graph for viewpoint selection and sequencing. 
%



More recently, coverage-informed sampling~\cite{fu2021computationally} was applied to the \iris inspection-planning framework~\cite{fu2023asymptotically}, using a randomized sparsification rule that myopically retains sampled configurations contributing new POI coverage. However, none of these approaches jointly consider motion geometry and inspection information, while providing an explicit guarantee on the quality of the resulting plans. 

This work therefore asks: \emph{Can an inspection roadmap be sparsified in a task-aware manner while preserving dense-roadmap coverage and providing an explicit guarantee on the quality of the inspection plans it represents?}

\begin{figure*}[th]
    \centering
    \subfloat[\texttt{Water-Tower}: dense roadmap]{        \includegraphics[width=.235\textwidth]{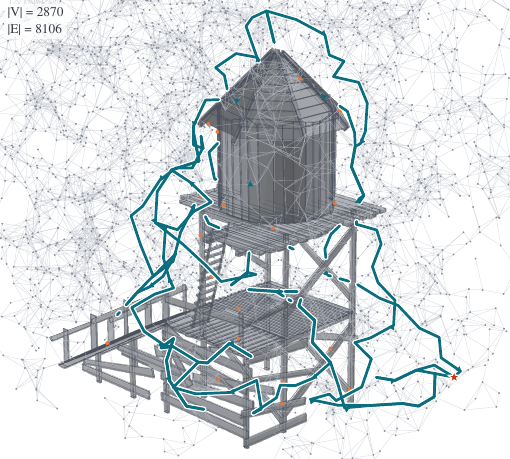}}
    \hfill%
    \subfloat[\texttt{Water-Tower}: \ispars roadmap]{       \includegraphics[width=.235\textwidth]{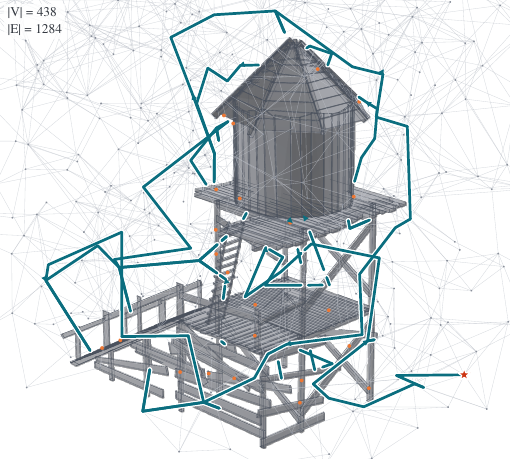}}%
    \hfill%
    \subfloat[\texttt{Bridge}: dense roadmap]{%
        \includegraphics[width=.235\textwidth]{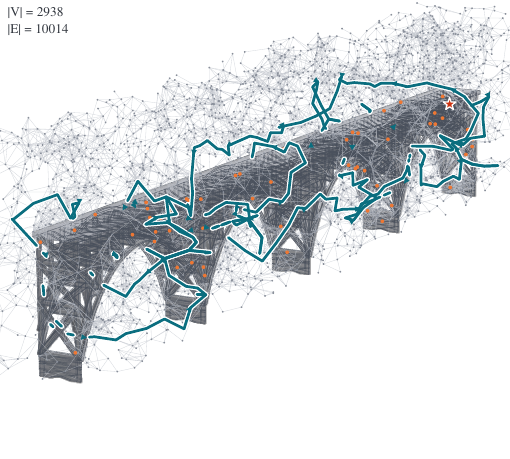}}%
    \hfill%
    \subfloat[\texttt{Bridge}: \ispars roadmap]{%
        \includegraphics[width=.235\textwidth]{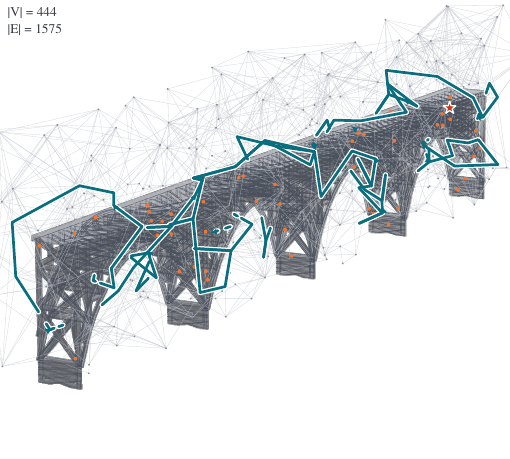}}%

    \par\vspace{0.4em}

    \includegraphics[width=.75\textwidth]{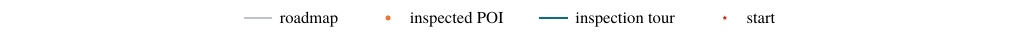}
    \caption{Inspection planning in 3D scenes, comparing dense roadmaps against sparse roadmaps generated by \ispars. Discarding redundant spatial detail while preserving geometric and POI-coverage information yields roadmaps on which good inspection plans are easier to find. 
    \kiril{also, report the final solution cost on top of each figure}
    } 
    \label{fig:roadmap-const-illustration}
\end{figure*}

\niceparagraph{Contribution.} We present \ispars, 
which is, to our knowledge, the first inspection-roadmap sparsifier with an explicit POI coverage preservation guarantee and tour-quality bound relative to a dense reference roadmap. Our method, described in Sec.~\ref{sec:method}, augments \spars with a local inspection-coverage admission mechanism that performs geometric roadmap sparsification while retaining local substitutes for inspection viewpoints.
The result is a substantially smaller plan search space, which still contains high-quality inspection plans.
Empirically, \ispars reduces roadmap vertex and edge counts by factors of $4-8\times$ across realistic inspection environments, as demonstrated in Sec.~\ref{sec:eval}.
Within the same solver budget, we show these reductions improve found tour lengths by up to $25\%$.

\section{Problem Formulation and Background}
\subsection{Inspection Planning}
An inspection planning instance is a tuple
$(\mathcal C, \mathcal C_{\mathrm{free}}, r, I,\chi)$, where $\mathcal C$ is the robot's configuration space,
$\mathcal C_{\mathrm{free}}\subseteq \mathcal C$ the collision-free subset, $r \in \mathcal C_{\mathrm{free}}$ the robot's departure configuration, $I$ is a discrete set of POIs, and $\chi:\mathcal C\rightarrow 2^I$ maps each configuration to the set of POIs it inspects.

Let $\mathcal L(q,q'):[0,1]\to \mathcal C$ denote the motion produced by a local planner (e.g., a straight-line path) between $q,q'\in \mathcal C$, when such a connection exists. 
%
%
%
%
An inspection tour is a sequence
\(\tau=(q_0,\dots,q_m)\), with \(q_0=q_m=r\), such that \(\forall i:q_i\in\mathcal C_{\mathrm{free}} \;\land \;\mathcal L(q_i,q_{i+1})
\subseteq \mathcal C_{\mathrm{free}}
\).
The tour length and coverage are defined, respectively, as
\[
c(\tau)
=
\sum_{i=0}^{m-1}
\left|\mathcal L(q_i,q_{i+1})\right|,
\qquad
\chi(\tau)
=
\bigcup_{i=0}^{m}\chi(q_i), 
\]
where $|\mathcal L(q_i,q_{i+1})|$ denotes the (arc) length of $L(q_i,q_{i+1})$.
The objective is to find a shortest collision-free tour that observes
every POI:
\[
\tau^\star
\in
\arg\min_{\tau}
\{
c(\tau)
\mid
\chi(\tau)=I
\}.
\]

Common approaches for inspection planning rely on sampling-based motion planning~\cite{kavraki1996probabilistic}, which approximates the robot's continuous configuration space using a weighted roadmap \(G=(V,E,c)\), where
\(V\subseteq \mathcal C_{\mathrm{free}}\) is a finite set of sampled configurations and \(E\subseteq V\times V\) represents collision-free local transitions (with respect to $\mathcal{L}$). 
Each edge \((u,v)\in E\) has weight
\(c(u,v)=\left|\mathcal L(u,v)\right|\),
and each vertex \(v\in V\) is associated with the observation set \(\chi(v)\). 
Restricting the inspection tours defined above to closed walks on \(G\) yields the \emph{graph inspection planning} (GIP) problem.

\subsection{The SPARS Framework}
In this work, we build upon the \spars framework~\cite{dobson2014sparse}, which constructs a compact motion-planning roadmap while preserving connectivity and bounded path quality. 
It simultaneously maintains a dense roadmap \(G_d=(V_d,E_d)\) and a sparse roadmap \(G_s=(V_s,E_s)\). Following a $\delta$-PRM roadmap algorithm, every sampled collision-free configuration is inserted into \(G_d\) and connected to existing vertices within radius \(\delta>0\) whenever the corresponding local path is collision-free. In contrast, a sample is admitted into \(G_s\) only when it serves a specific geometric purpose.

The key theoretical advantage of \spars is the \emph{transfer of guarantees} between the two roadmaps. The dense roadmap carries the burden of asymptotically approximating optimal paths in $\mathcal C_{\mathrm{free}}$, while the sparse roadmap is required only to preserve the geometrically significant paths revealed by the dense roadmap, up to a bounded length increase. As $G_d$ converges with the number of PRM samples toward optimal paths, this approximation guarantee transfers to $G_s$. We next briefly review the mechanisms underlying this construction. Full details are provided in~\cite{dobson2014sparse}. 

\niceparagraph{Visibility regions and representatives.}
%
Given a maximal local-plan range $\Delta > \delta$, define the set of sparse vertices reachable from a configuration $q$ as:
\[ W(q)= \{ g\in V_s \mid \ \mathcal L(q,g)\subset C_{\mathrm{free}} \land | \mathcal L(q,g)|\leq\Delta\}. \]
The representative of $q$ in $G_s$ is the nearest such vertex:
$\operatorname{rep}(q)=\text{argmin}_{g\in W(q)} |\mathcal L(q,g)|$. 

\niceparagraph{Admission criteria.}
Each sample \(q\) from the dense roadmap \(G_d\) is evaluated by four
criteria, each of which may augment the sparse roadmap \(G_s\)
(Fig.~\ref{fig:spars-illust}).
First, \(q\) is admitted as a \emph{Guard} if \(W(q)=\emptyset\).
That is, if no sparse vertex is \(\Delta\)-reachable from \(q\).
Second, \(q\) is admitted as a \emph{Bridge} if \(W(q)\) contains
vertices $q', q''$ belonging to distinct connected components of \(G_s\), in
which case edges \((q',q)\) and \((q,q'')\) are added to $G_s$ as well.
Otherwise, \(q\) may trigger the \emph{Interface} criterion if it is
locally connectable to two nearby sparse vertices \(q',q''\), that are
already connected in \(G_s\) but do not share an edge.
These local connections reveal an interface between their visibility regions.
The algorithm first attempts to add the direct edge \((q',q'')\). If
this connection is obstructed, it admits \(q\) and adds the edges
\((q',q)\) and \((q,q'')\).
The final (iv) \emph{Shortcut} criterion uses \(G_d\) to identify short paths between configurations supporting two interfaces of \(\operatorname{rep}(q)\).
It compares each such path with the corresponding path between interface midpoints in \(G_s\).
When the prescribed stretch factor \(t\) is violated, the criterion first attempts to add a direct edge between the relevant sparse vertices. Otherwise, it inserts suitable interface-supporting configurations and a smoothed path derived from \(G_d\).
This ensures $G_s$ paths approximate paths on $G_d$ by a $t$-factor inside any visibility region, providing \spars's path length approximation bound. 
The \spars pseudo-code is presented in the non-colored lines of Alg.~\ref{alg:inspection-spars}. 

\begin{figure}[t]
    \centering
    \begin{minipage}[c]{0.5\linewidth}
        \centering
        \includegraphics[width=\linewidth]{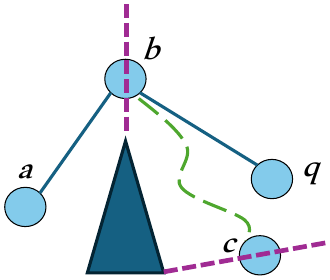}
    \end{minipage}
    \hspace{1em}
    \begin{minipage}[c]{0.37\linewidth}
        \vspace{1em}
        \caption{Illustration of the four SPARS admission criteria:
        \(a\) is added as a guard, $b$ may be added as a bridge or to support an interface connection, and \(c\) may be added to preserve a
        dense-roadmap shortcut whose sparse detour violates the
        prescribed stretch factor.}
        \label{fig:spars-illust}
    \end{minipage}%
    \vspace{-1em}
\end{figure}


    
\niceparagraph{Approximation guarantee.}
Let \(d_{\mathcal C}(q,q')\) denote the length of a shortest \emph{continuous} collision-free path between \(q,q'\in\mathcal C_{\mathrm{free}}\).
Let \(d_s(q,q')\) denote the length of the shortest route through \(G_s\), including the connections to sparse representatives.
Under standard sampling and positive-clearance assumptions made by \spars, the resulting sparse roadmap asymptotically satisfies the relation
\begin{equation}
    d_s(q,q')
    \leq
    t\cdot d_{\mathcal C}(q,q')+4\Delta
    \qquad
    \forall q,q'\in\mathcal C_{\mathrm{free}}.
    \label{eq:spars-bound}
\end{equation}

\section{Coverage-Preserving Sparse Roadmaps}
\label{sec:method}
In this section, we introduce our inspection-aware sparsification approach that preserves the quality of the overall inspection plan. In preparation, we first discuss the limitations of geometric sparsification (i.e., agnostic to the inspection task) and global POI-coverage preservation (which only considers whether a POI is covered without regard to the induced solution quality). 

\subsection{Limitations of Global Coverage Preservation}
For a roadmap \(G\), let
\(
I(G)=\bigcup_{v\in V(G)}\chi(v)
\)
denote its POI coverage. Next, we demonstrate that vanilla \spars does not ensure that the set of inspected POIs is preserved in the sparse roadmap, i.e., \(I(G_s)=I(G_d)\). 
As illustrated in Fig.~\ref{fig:counterex1}, suppose that a configuration \(q\) is admitted as a guard, and a subsequently sampled configuration \(q'\) observes a POI \(p\) unobserved by \(q\), but lies within the SPARS visibility region of \(q\).
If \(q'\) satisfies no other geometric admission criterion, it is retained in \(G_d\) but discarded from \(G_s\), thereby losing the observation of \(p\). Multiple occurrences of this construction can cause
\(G_s\) to omit an arbitrarily large fraction of the POIs represented in \(G_d\).

A natural repair is a \emph{global coverage rule} that admits a sample \(q\) whenever
\(
\chi(q)\setminus I(G_s)\neq\emptyset.
\)
Although this rule preserves \(I(G_s)=I(G_d)\) during construction, it may discard useful inspection-plan vertices that can lead to a lower solution cost.
Consider a POI \(p\) observable from two spatially separated configurations \(q\) and \(q'\), as illustrated in
Fig.~\ref{fig:counterex2}.
If \(q\) is admitted first, it globally represents \(p\), allowing the rule to reject \(q'\). Nevertheless, a dense-roadmap tour may use \(q'\) to cover $p$, while visiting \(q\) requires an arbitrarily long detour. Hence, despite having identical coverage, a sparse graph may reveal subpar inspection plans.

\begin{figure}
\vspace{0.5em}
\centering
\subfloat[\label{fig:counterex1}]{
\includegraphics[width=0.45\linewidth]{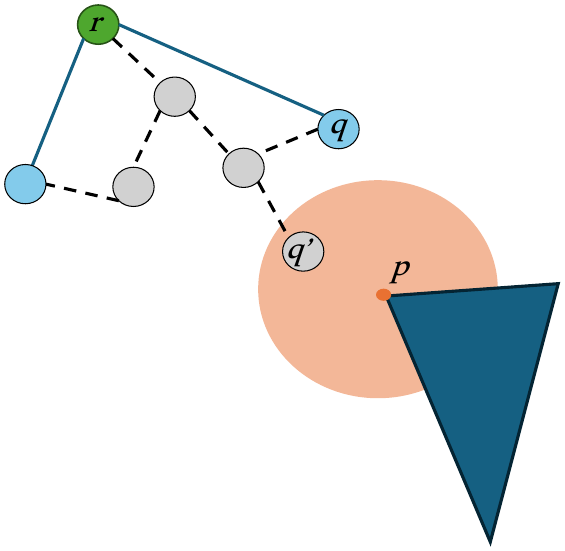}
}\hfill
\subfloat[\label{fig:counterex2}]{
\includegraphics[width=0.45\linewidth]{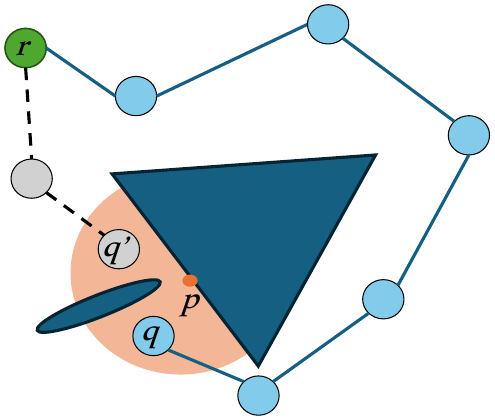}
}
\vspace{-0.5em}
\caption{Examples motivating local inspection-coverage preservation. POIs are in orange, shaded area is visibility range. Dense-graph edges are gray and sparse-graph edges are light blue. (a) Geometric sparsification without POI-coverage enforcement may retain $q$ and omit $q'$, losing inspection on $p$. (b) Global POI-coverage enforcement may choose $q$ over $q'$, leading to a lower-quality solution.}
\label{fig}
\vspace{-0.5em}
\end{figure}

\subsection{Local inspection-coverage invariant}
The preceding example shows that global coverage preserves inspection feasibility but provides no guarantee on inspection-plan quality. In particular, \(I(G_s)=I(G_d)\) ensures that every POI is observed somewhere in the sparse roadmap, yet it does not relate the dense-roadmap inspecting vertex with a spatially near sparse-roadmap inspecting vertex.
Consequently, a dense tour may inspect a POI \(p\) at a configuration \(q\) already lying along its route, while every sparse configuration observing \(p\) lies arbitrarily far from \(q\), forcing a potentially unbounded detour. 
To prevent this, we require every observation available at a dense-roadmap vertex to be represented locally in the sparse roadmap.
We introduce the following invariant to formalize this coverage requirement.

\begin{invariant}[Local inspection coverage]
\label{def:insp-invariant}
For every vertex $q\in V(G_d)$ and every POI $p\in\chi(q)$,
there exists a vertex $q'\in V(G_s)$ such that:
\[
    p\in\chi(q'),\qquad
    \mathcal L(q,q')\subset \mathcal C_{\mathrm{free}},\qquad
    |\mathcal L(q,q')|\leq\Delta .
\]
\end{invariant}

For a roadmap-sparsification algorithm maintaining this invariant, every observation available at a dense-roadmap vertex is represented by an observing sparse-roadmap vertex reachable through a valid local path of length at most $\Delta$ from $q$. We emphasize that different POIs $p\in \chi(q)$ may be covered by different vertices $q'\in G_s$. 


\subsection{Inspection-SPARS}
Following the invariant definition, we introduce \ispars ---roadmap sparsification for inspection planning (Alg.~\ref{alg:inspection-spars}).
\ispars augments \spars with a local inspection admission mechanism, presented in the \text{\color{isparsblue} colored lines}.
This mechanism compares the POI-coverage of the candidate $q$ with the POI-coverage in its reachable sparse-graph geometric neighborhood. The vertex $q$ is admitted whenever it observes any POI that is not observed by locally reachable sparse-graph vertices.
Details of all subroutines can be found in the \spars paper~\cite{dobson2014sparse}.

When a candidate \(q\) passes the inspection test, it is inserted into \(G_s\) using the standard \(\textsc{AddGuard}\) operation. This operation adds \(q\) as a sparse vertex but introduces no task-specific edges. Once admitted, \(q\) participates in all subsequent SPARS tests exactly like any other sparse vertex: it may serve as a reachable guard or representative, and edges incident to it are introduced only through the original bridge, interface, and shortcut mechanisms. 
Task information therefore affects only vertex retention, increasing the number of admitted vertices, while vanilla \spars machinery remains responsible for the roadmap's geometric structure. 
In Sec.~\ref{sec:eval-construction} we show this leads to \ispars having only a mild increase in roadmap size over \spars.

We place the inspection criterion after the first three \spars criteria and before the fourth, shortcut criterion. This design aims for computational efficiency, testing cheaper conditions earlier. 

\begin{algorithm}
\caption{\ispars roadmap construction.
\text{\color{isparsblue} Colored} lines constitute the inspection extension.}
\label{alg:inspection-spars}
\begin{algorithmic}[1]
\STATE Initialize $G_d, G_s$ with a single vertex $r$. 

\WHILE{$|G_d|\leq N$ \label{alg:line:while}}
    \STATE $q\gets \textsc{SampleConfiguration}(\mathcal C_{\text{free}})$
    \STATE $V_d.\text{insert}(q)$
    \FORALL{$x \in V_d \setminus \{q\}$ \textbf{s.t.} $|\mathcal{L}(x,q)| \leq \delta$}
        \IF{$\mathcal{L}(x,q) \subset \mathcal{C}_{\text{free}}$}
            \STATE $E_d.\textsc{insert}(x,q)$
        \ENDIF
    \ENDFOR
    \STATE $\mathcal W \gets \textsc{ReachableGuards}(q,\Delta, G_s)$
    \IF{$\mathcal W==\emptyset$}
        \STATE \textsc{AddGuard}$(q, G_s)$
        \ELSIF{$\exists w, w'\in \mathcal W$ disconnected in $G_s$}
        \STATE \textsc{AddBridge}$(q, \mathcal W, G_s)$
        \ELSE
        \STATE $\textsc{AddInterface}(q,\Delta,G_s)$
    \ENDIF

    {\color{isparsblue}
    \IF{$q\notin V(G_s)$}
        \STATE $I_\Delta(q,G_s)\gets\bigcup_{w\in\mathcal W}\chi(w)$
        \IF{$\chi(q)\setminus I_\Delta(q,G_s)\neq\emptyset$}
            \STATE \textsc{AddGuard}$(q, G_s)$
        \ENDIF
    \ENDIF
    }

    \IF{$q\notin V(G_s)$}
        \STATE\textsc{AddShortcut}$(q,t,G_d,G_s)$
    \ENDIF
\ENDWHILE


\RETURN $G_s,G_d$
\end{algorithmic}
\end{algorithm}

\section{Theoretical Analysis}
\label{sec:analysis}
We now turn to a formal analysis of the guarantees of the \ispars approach. 
As in the analysis for \spars~\cite{dobson2014sparse} and \iris~\cite{fu2023asymptotically}, we assume that the approximated continuous inspection plan is \emph{robust}, i.e., it can be approximated arbitrarily well using a sample set whose measure is strictly positive.  

\subsection{Properties of Inspection-SPARS} 
We first establish that \ispars preserves both the inspection information represented by the dense roadmap as formalized by Inv.~\ref{def:insp-invariant}, and the geometric properties inherited from \spars. 

\begin{claim}[Inspection invariant preservation] \label{clm:invariant-preservation} 
\ispars maintains the local inspection-coverage Invariant~\ref{def:insp-invariant} after every construction iteration (line~\ref{alg:line:while}). 
\end{claim} 

\begin{proof}
We prove by induction over the sampled configurations. 
Suppose the invariant holds before processing \(q\). 
If \(q\) is admitted into
\(G_s\), then \(q\) itself represents every \(p\in\chi(q)\). 
Otherwise, $q$ has not passed the \text{\color{isparsblue}inspection admission test}, which implies that 
\(
    \chi(q)\subseteq I_\Delta(q,G_s)
\).
Hence, for every \(p\in\chi(q)\), some
\(g\in \mathcal W\) observes \(p\). 
By definition, this \(g\) is connected to \(q\) by a collision-free local path of length at most \(\Delta\). 
Thereafter, any such representative remains valid, as no sparse vertices are removed during construction. 
The claim follows by induction.
\end{proof}


Claim~\ref{clm:invariant-preservation} concerns the coverage of the complete roadmaps. At finite \(N\), however, the root-connected components of the roadmaps \(G_d[r]\) and \(G_s[r]\) may cover different POI sets because the connectivity guarantees inherited from \spars are asymptotic. This distinction motivates restricting the subsequent comparison to POIs reachable from the root in both roadmaps. 


Next, we address the preservation of \emph{geometric} properties by \ispars, showing the added mechanism does not asymptotically break the guarantees provided by \spars.

\begin{claim}[Preservation of \spars geometric guarantees] 
\label{lem:spars-preservation}
\ispars preserves the \spars configuration-space representation, connectivity, and path-quality guarantees.
\end{claim} 

\begin{proof}[Proof sketch]
The claim follows from the asymptotic analysis of \spars~\cite{dobson2014sparse},
combined with the fact that the inspection test fires only finitely many times.

Consider a sample $q$ admitted by the inspection test, triggered by some POI $p$.
Once $q\in V(G_s)$, no later sample that can reach $q$ by a collision-free local
path of length at most $\Delta$ is admitted by the inspection test due to the same $p$.
Hence all vertices admitted due to $p$ are pairwise not $\Delta$-reachable, and their number is finite under the positive-clearance assumptions of \spars~\cite{dobson2014sparse}. As $I$ is
also finite, the inspection test fires only finitely many times.

Before its last firing, the inspection test can preempt only the shortcut test,
as the guard, bridge, and interface tests precede it. 
However, each skipped shortcut event is witnessed by a nonzero-measure set of samples, and is therefore asymptotically guaranteed to be processed by a later sample. 

After the last firing, \ispars executes exactly the \spars procedure on a roadmap augmented by finitely many ordinary sparse vertices. These vertices may refine the
visibility-region decomposition and update representatives, as guard
admissions do in \spars, but vertices and edges are never removed, so free-space coverage and connectivity are retained. The geometric guarantees
of \spars therefore follow.
\end{proof}



\subsection{Inspection-Planning Length Approximation}
\label{sec}
Having established local preservation of dense-roadmap observations (Claim~\ref{clm:invariant-preservation}) and retention of the \spars path-quality properties (Claim~\ref{lem:spars-preservation}), we apply them to compare inspection tours represented by the two roadmaps.

Let
\(
I_d= I(G_d[r]),
I_s=I(G_s[r])
\)
be the reachable POI sets by the corresponding roadmaps,
and fix a common target set of POIs-to-inspect \(I_t= I_d\cap I_s\), with \(K=|I_t|\). Let \(\tau_d^\star\) and \(\tau_s^\star\) denote shortest tours inspecting \(I_t\) on \(G_d\) and \(G_s\), respectively.

\begin{theorem}[Inspection-tour approximation] \label{thm:inspection-approximation} As the number of samples \(N\) tends to infinity, with probability approaching one, \[ c(\tau_s^*) \leq t\,c(\tau_d^*) +2t\cdot \Delta \cdot K. \] 
\end{theorem} 

\begin{proof} 
Let $\tau_d^*=(q_0,q_1,\ldots,q_m)$ be an optimal inspection plan on the dense graph $G_d$, with $q_0=q_m=r$. Applying the \spars path bound~\eqref{eq:spars-bound} separately to
every edge or vertex transition of \(\tau_d^*\) could accumulate to an excessively large additive term proportional to the number of transitions.

We therefore partition the tour according to \emph{coverage events}.
For every \(p\in I_t\), let \(x_d(p)\) be the first vertex along \(\tau_d^*\) that inspects \(p\).
Order the POIs as \(p_1,\dots,p_K\) according to the occurrence of \(x_d(p_i)\) along
the tour, breaking ties arbitrarily, and define $\forall i: x_d^i=x_d(p_i)$. We set $x_d^0=x_d^{K+1}=r$. These vertices partition \(\tau_d^*\) into \(K+1\) consecutive walk segments. Let \(\tau_d^i\) denote the segment from \(x_d^i\) to \(x_d^{i+1}\), and let $\ell_i=c(\tau_d^i)$.
Note that a segment may be empty when multiple POIs are first inspected by the same vertex. 
As the segments form a partition of \(\tau_d^*\), we have that 
\(\sum_{i=0}^{K}\ell_i=c(\tau_d^*)\).

By the local inspection-coverage Invariant~\ref{def:insp-invariant}, for every \(i\in[1,K]\) there exists a sparse vertex
\(x_s^i\in V(G_s)\) that observes \(p_i\) and is connected to \(x_d^i\) by a collision-free local path of length at most \(\Delta\). 
On the asymptotic connectivity event established above, every such \(x_s^i\) belongs to
\(G_s[r]\).

We construct the following walk on the sparse graph: \(\widetilde{\tau}_s\) is a concatenation of shortest paths in \(G_s\) between each consecutive pair \(x_s^i,x_s^{i+1}\). 
As before, we set \( x_s^0=x_s^{K+1}=r\). 
Since \(x_s^i\) inspects \(p_i\), \(\widetilde{\tau}_s\) is a valid inspection plan for \(I_t\).

Consider one such segment \(x_s^i,x_s^{i+1}\). 
A collision-free \emph{continuous} path between those endpoints can be constructed by concatenating three motion segments: The local planner path \(\mathcal L(x_s^i,x_d^i)\), the dense-graph subwalk \(\tau_d^i\), and the local planner path \(\mathcal L(x_d^{i+1},x_s^{i+1})\).
Consequently,
\(
    d_{\mathcal C}(x_s^i,x_s^{i+1})
    \leq
    \ell_i+2\Delta
\).


Next, we upper-bound the distance between $x_s^i,x_s^{i+1}$ over~$G_s$. 
For endpoints already in the sparse roadmap, we refine the \spars path-decomposition argument to obtain a purely multiplicative bound.
For general $q, q'$, the additive term $2\Delta$ (per endpoint) has two sources.
The first $\Delta$ accounts for the local path from $q$ to its representative.
This vanishes here, since $x_s^i,x_s^{i+1}$ are their own representatives. 
The second $\Delta$ accounts for the representative-to-interface segment, which may have length up to $\Delta$ from $\operatorname{rep}(q)$, but be infinitesimally short from $q$ itself. 
Again, since $\operatorname{rep}(x)=x$, these two distances coincide.
\footnote{We note that this refinement is only used for tightening the bound, which holds with a slightly looser additive term under the general \spars path-length bound.} 
Thus, the sparse-endpoint stretch property gives
\(
    d_s(x_s^i,x_s^{i+1})
    \leq
    t\cdot d_{\mathcal C}(x_s^i,x_s^{i+1})
\), 
and plugging $d_{\mathcal C}(x_s^i,x_s^{i+1})$ gives 
\(
d_s(x_s^i,x_s^{i+1})\leq t\ell_i+2t\Delta
\).
Note also that for the first and last segments, as \(x_s^0,x_d^0,x_s^{K+1},x_d^{K+1}\) all equal $r$, one of the accounted local paths vanishes.
Summing the segment bounds, the two endpoint segments contribute
\(t\Delta\) each and the \(K-1\) internal segments contribute
\(2t\Delta\) each. Hence,
\[
\begin{aligned}
c(\widetilde{\tau}_s)
&\leq t\sum_{i=0}^{K}\ell_i+2t\Delta K
=t\,c(\tau_d^\star)+2t\Delta K .
\end{aligned}
\]

Finally, we look at \(\tau_s^*\)---an optimal inspection plan on
\(G_s\)---whereas \(\widetilde{\tau}_s\) is one feasible such plan. Therefore,
\[
c(\tau_s^*)
\leq c(\widetilde{\tau}_s)
\leq
t\,c(\tau_d^*)+2t\Delta\cdot K. \qedhere
\]
\end{proof}

Under the standard robustness assumptions for sampling-based inspection
planning~\cite{fu2023asymptotically}, the optimal dense-roadmap tour converges
in probability to an optimal continuous-space inspection tour $\tau^\star$.
Combined with Theorem~1, the optimal sparse-roadmap tour therefore
asymptotically satisfies $c(\tau^\star_s) \le t\,c(\tau^\star) + 2t\Delta K$.

The additive term of this bound grows linearly with $K$, and may become large
for instances with many POIs.
In practice, however, the gap between sparse and dense tours is typically far smaller than the bound suggests. 
First, the additive term thus effectively scales with the number $K' \le K$ of \emph{distinct} representatives requiring a detour.
Second, the underlying \spars path-quality bound is itself a conservative worst
case~\cite{dobson2014sparse}. Our evaluation has found the sparsification to improve inspection tours planned in all scenarios examined.

\section{Evaluations}
\label{sec:eval}

\begin{figure*}
    \centering
    \subfloat[\texttt{City} - dense roadmap]{     \includegraphics[width=.32\textwidth]{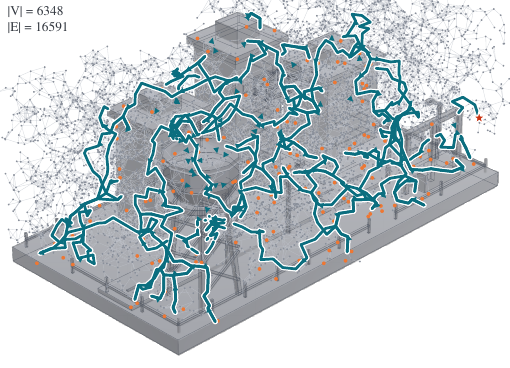}}
    \hfill%
    \subfloat[\texttt{City} - \ispars roadmap]{       \includegraphics[width=.32\textwidth]{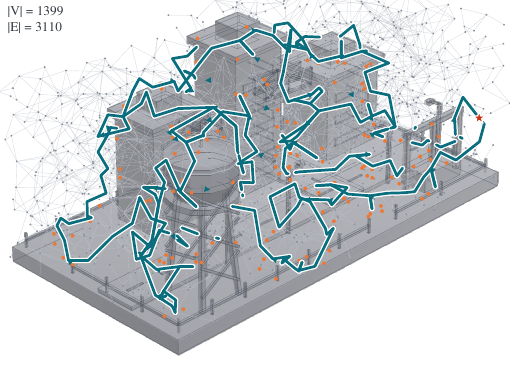}}%
    \hfill%
    \subfloat[\texttt{City} - \irisc roadmap]{%
        \includegraphics[width=.32\textwidth]{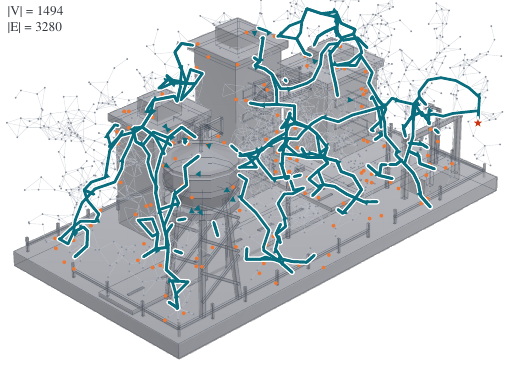}}%
    \hfill%
    \par\vspace{0.4em}

    \includegraphics[width=.75\textwidth]{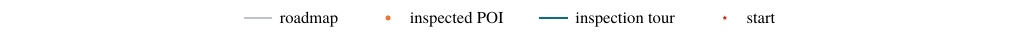}

    \caption{Roadmap and solution inspection plan on the \texttt{City} scene. 
    \kiril{also, report the final solution cost on top of each figure}
    } 
    \vspace{-1em}
    \label{fig:inspection-city}
\end{figure*}

In this section, we confirm the principal end-to-end claim of this work: \ispars enables better inspection plans within a fixed optimization budget. 
We then proceed to analyze our method through an exploration of constructed roadmap characteristics.

\subsection{Experimental Setup}
We evaluate \ispars in three different scenes: \texttt{Water-Tower}, \texttt{Bridge} (Fig.~\ref{fig:roadmap-const-illustration}), and \texttt{City} (Fig.~\ref{fig:inspection-city}).
For each scene, obstacles are represented by triangular meshes, and POIs are sampled uniformly at random from the objects’ surfaces\footnote{No mechanism was used to ensure that every sampled POI is observable from a feasible robot configuration. Thus, \(K\) denotes the number of sampled POIs and may exceed the number of inspectable POIs.}.
The robot is assigned a fixed root configuration where the tour begins and ends.

We model the robot as a spherical aerial vehicle translating in a 3D configuration space.
This provides a simple abstraction of aerial inspection~\cite{cao2025cooperative}, in which a camera sensor is independently oriented using a gimbal and is not considered part of the robot's configuration space.
Accordingly, a POI is visible from a configuration when it lies within the \emph{sensor range} and the connecting line segment is not occluded by the scene mesh.
We consider geometric planning only, without vehicle dynamics. Nevertheless, \ispars applies to general configuration spaces, provided suitable sampling, distance, local-planning, collision-checking, and visibility operations are available.

\niceparagraph{Comparison protocol.}
Three roadmap construction methods are evaluated for planning comparison. (i) \ispars, our proposed method, (ii) its $\delta$-PRM dense roadmap \dense as a no-sparsification baseline, and (iii) \irisc as a state-of-the-art roadmap for sampling-based inspection planning.
\irisc denotes the coverage-informed roadmap construction introduced for \iris~\cite{fu2021computationally}. It constructs an RRG~\cite{karaman2010incremental} while always admitting a candidate that contributes previously unseen POI coverage, and otherwise admitting it with probability \(p_{\mathrm{accept}}\).\footnote{We evaluate \irisc only as roadmap-construction method, independently of the \iris graph-search algorithm.}

The different roadmaps are subsequently passed to the GIP-solving framework of Morgan et al.~\cite{morgan2026scalable}. 
Before invoking the inspection solver, each roadmap is restricted to its root-connected component \(G[r]\). To ensure that all methods solve exactly the same inspection
task, for each seed we define the target set as
\(
 I_t =
 I(G_d[r])
 \cap I(G_s[r])
 \cap I(G_{iris-c}[r]).
\)

This protocol prevents differences in reachable POI coverage from confounding the tour-length comparison. In particular, POIs reached by \dense and \ispars but missed by \irisc are excluded, making the comparison conservative with respect to \irisc. Coverage preservation is evaluated separately in Sec.~\ref{sec:eval-construction}.

\niceparagraph{Implementation.}
Experiments are performed on a laptop  with an Intel Core Ultra~9 185H CPU and 64\,GB RAM.
We implemented all evaluated methods in Python~3.12 and used identical collision-checking, visibility evaluation, and roadmap operations. Our \irisc implementation follows the official reference implementation accompanying~\cite{fu2021computationally}. 
%
All our source code, experiment scenes, and configurations are available in an  \href{https://anonymous.4open.science/r/IS-submission-85AF}{anonymous repository}\footnote{https://anonymous.4open.science/r/IS-submission-85AF}.

\niceparagraph{Planner parameters and calibration.}
Table~\ref{tab:planner-parameters} summarizes the roadmap-construction parameters used in the different experiments. For \spars-based methods, \(\delta\) is the connection radius of the dense roadmap, while \(\Delta\) is the maximum local-planner range used to construct the sparse roadmap. We select these values according to the geometric scale of each scene, selecting \(\Delta=1.5\times\delta\). Increasing \(\Delta\) generally promotes stronger sparsification by enlarging the locally represented region, but also increases the additive term in the path- and
inspection-tour bounds. We fix the \spars stretch factor to \(t=1.05\) in all experiments. This limits the multiplicative stretch term to \(5\%\), independently of the additive dependence on \(\Delta\).
For a geometrically comparable evaluation, \irisc uses step size \(\eta=\delta\) and RRG connection radius
\(r_{\mathrm{RRG}}=\Delta\). 
We calibrate \(p_{\mathrm{accept}}=0.2\) so \irisc approximates the edge
sparsification achieved by \ispars. 
This value is within the typical range for \irisc~\cite{fu2021computationally}. All methods use the same local planner and collision checker. 

\begin{table}[t]
    \centering
    \caption{Geometric and sensing parameters. The
    \iris-based methods use \(\eta=\delta\) and
    \(r_{\mathrm{RRG}}=\Delta\). }
    \label{tab:planner-parameters}
    \footnotesize
    \setlength{\tabcolsep}{3.8pt}
    \begin{tabular}{lcccc}
        \toprule
        Scene &
        Robot radius &
        Sensor range &
        $\delta\;(=\eta)$ &
        $\Delta\;(=r_{\mathrm{RRG}})$ \\
        \midrule
        \texttt{Water-Tower} & 0.5 & 6 & 10 & 15 \\
        \texttt{Bridge}      & 0.8 & 10 & 3 & 4.5   \\
        \texttt{City}      & 0.8 & 10 & 8 & 12 \\
        \bottomrule
    \end{tabular}
    \vspace{-2em}
\end{table}

\subsection{End-to-End Inspection-Planning Performance}

\begin{figure*}
    \vspace{0.5em}
    \centering
    \includegraphics[width=\linewidth]{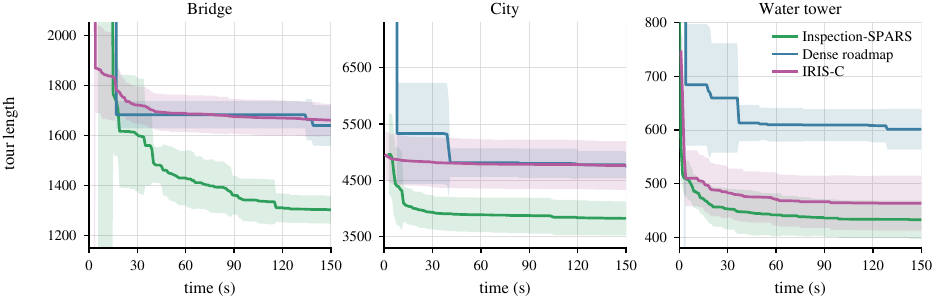}
    \caption{Anytime solver performance with respect to inspection plan length. Solid curves show the mean over the six seeds, and shaded regions show one standard deviation. Lower is better.
    }
    \label{fig:solver-performance}    
    \vspace{-1em}
\end{figure*}

In the following experiment, we examine whether task-aware roadmap sparsification enables the inspection solver to find shorter tours within a fixed computation budget.
We consider three scenarios of different scales: \texttt{Water-Tower} with \(N=4000\) collision-free samples and \(K=100\) generated POIs, \texttt{Bridge} with \(N=5000\) and \(K=500\), and \texttt{City} with \(N=7000\) and \(K=500\). 
For each scene we construct \dense,
\ispars, and \irisc roadmaps, and for every roadmap, the GIP solver is allotted a fixed optimization budget of \(150\,\mathrm{s}\), and results are reported over six independent random seeds.

\begin{table}[t]
\centering
\caption{Roadmap statistics. Values are mean $\pm$ standard deviation.}
\label{tab:roadmap-statistics}

\scriptsize
\setlength{\tabcolsep}{2.5pt}
\renewcommand{\arraystretch}{1.05}

\begin{adjustbox}{max width=\columnwidth}
\begin{tabular}{@{}llrrr@{}}
\toprule
Scenario & Roadmap & $|V|$ & $|E|$ & Covered POIs \\
\midrule
\multirow{3}{*}{\texttt{Bridge}}
 & \dense   & $4955 \pm 11$ & $27508 \pm 457$ & $404.2 \pm 6.1$ \\
 & \ispars & $875 \pm 18$  & $3589 \pm 42$   & $404.8 \pm 6.1$ \\
 & \irisc  & $1199 \pm 28$ & $3935 \pm 103$  & $381.7 \pm 23.0$ \\
\midrule
\multirow{3}{*}{\texttt{City}}
 & \dense   & $6391 \pm 40$ & $16891 \pm 281$ & $225.5 \pm 11.6$ \\
 & \ispars & $1425 \pm 29$ & $3165 \pm 61$   & $249.8 \pm 8.7$ \\
 & \irisc  & $1462 \pm 48$ & $3101 \pm 224$  & $192.7 \pm 5.0$ \\
 \midrule
\multirow{3}{*}{\texttt{Water-Tower}}
 & \dense   & $3933 \pm 17$ & $14474 \pm 198$ & $73.8 \pm 4.1$ \\
 & \ispars & $476 \pm 10$  & $2086 \pm 14$   & $76.2 \pm 3.6$ \\
 & \irisc  & $843 \pm 24$  & $3168 \pm 165$  & $70.8 \pm 3.2$ \\
\bottomrule
\end{tabular}
\end{adjustbox}
\vspace{-2em}
\end{table}

Table~\ref{tab:roadmap-statistics} reports the resulting roadmap sizes and POI coverage. 
This table shows \ispars provides substantial sparsification over \dense, building $4-8\times$ fewer vertices and edges, and does so without any coverage loss. Coverage is even slightly improved by \ispars occasionally, due to the \emph{bridge} and \emph{interface} mechanisms allowing further connectivity. 
\irisc on the other hand, loses POIs coverage as expected. We further analyze this finding in Sec.~\ref{sec:eval-construction}, focusing on graph construction. $I_t$ is taken as the set of POIs covered by all roadmaps.

Before solving the GIP on the roadmaps, the employed GIP solver~\cite{morgan2026scalable} was configured to provide the best possible results for each roadmap, by selecting between the available underlying MILP formulations and their solution approaches: the \textsc{Group-Cutset} formulation was used for the larger \dense roadmap, and the \textsc{SCF} formulation for the smaller \ispars and \irisc roadmaps.\footnote{In all cases, both formulations were evaluated, and the best results were selected. In particular, using the \textsc{SCF} solver for \dense typically failed to produce a feasible solution within the available runtime. This solver formulation distinction can be seen as a direct advantage of roadmap sparsification---enabling the use of a complete \textsc{SCF} formulation rather than a lazy branch-and-cut-based \textsc{Group-Cutset} formulation.}
In all cases, the solver was allotted a fixed $150\,\mathrm{s}$ computational budget.

Fig.~\ref{fig:solver-performance} shows the best feasible tour length found by the solver, as a function of wall-clock time.
On the \texttt{Bridge} and \texttt{City} scenarios, tours found on \ispars are $20\%-25\%$ shorter than those found using other methods.
For the geometrically simpler \texttt{Water-Tower} scenario, while improvement over \dense is substantial, the gain over \irisc is smaller.

%

In Fig.~\ref{fig:inspection-city}, an instance of the \texttt{City} scene illustrates the structural differences among the roadmaps.
The \dense roadmap is heavily cluttered with structure that is largely redundant for inspection-plan quality.
The \irisc roadmap is sparser because it retains only a randomly thinned subset of the samples. Its limited spatial exploration and
tree-like structure lead to an inefficient tour with substantial backtracking.
By contrast, \ispars preserves the exploration and connectivity of \dense while removing redundancy and retaining routes and shortcuts
missed by \irisc, enabling the GIP solver to find high-quality tours rapidly.
A similar pattern appears in the \texttt{Bridge} scene, where \ispars uses the sampling budget to achieve well-distributed geometric
coverage, whereas the random thinning of \irisc results in weaker spatial exploration.
In the geometrically simpler \texttt{Water-Tower} scene, \irisc retains competitive tours even with naive sparsification, yielding tours close to those planned on \ispars and significantly better than tours planned on \dense.

\subsection{Roadmap Construction and Scalability}
\label{sec:eval-construction}
We next examine how roadmap structure scales with the sampling budget \(N\). Fig.~\ref{fig:construction} reports the numbers of vertices and edges, POI coverage, and cumulative collision checks in the \texttt{Water-Tower} scene over ten independent seeds.

\begin{figure}
    \centering
    \includegraphics[width=1.0\linewidth]{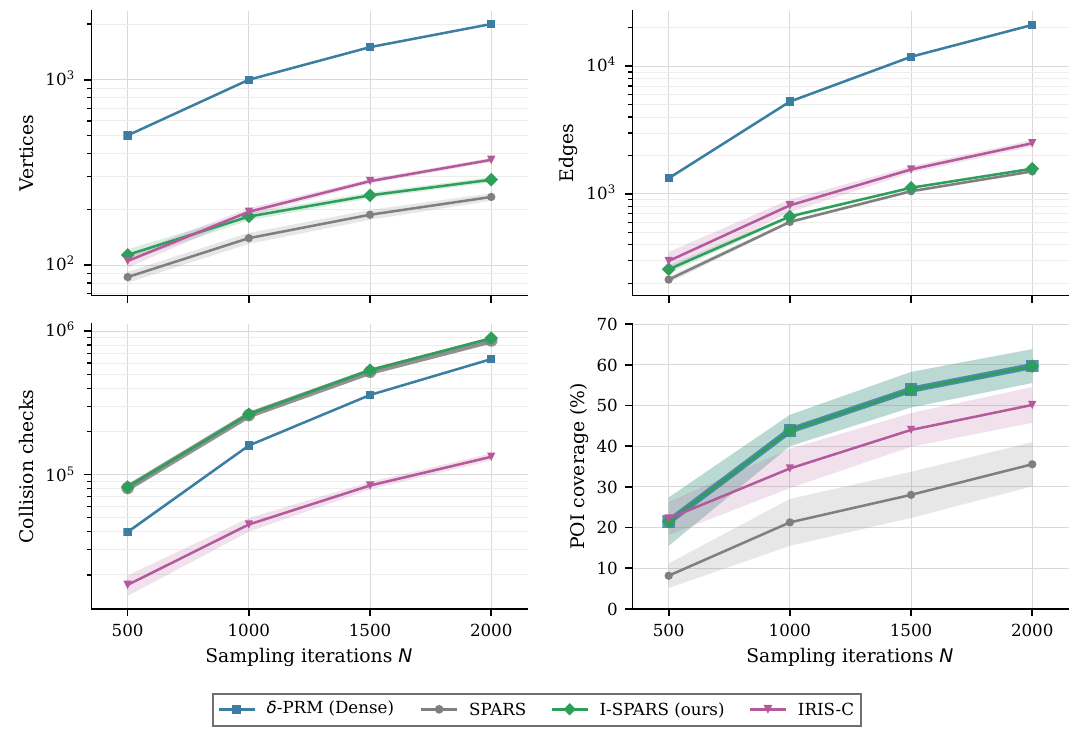}
    \caption{Roadmap construction vs. number of sampling iterations \(N\) for \texttt{Water-Tower}. Curves show means over ten random seeds, shaded regions show standard deviation. Count-based axes use logarithmic scales.}
    \label{fig:construction}
    \vspace{-2em}
\end{figure}

Across all sampling budgets, \ispars produces a substantially smaller roadmap than \dense, with the reduction being more pronounced as \(N\) increases. At \(N=2000\), it contains approximately \(6\times\) fewer vertices and \(8\times\) fewer edges than \dense. This is achieved while POI coverage remains identical to this of the \dense roadmap, with distinct advantage over \irisc and \spars.

Although Claim~\ref{clm:invariant-preservation} guarantees identical POI coverage for the complete \dense and \ispars roadmaps at every iteration, their root-connected components are guaranteed to agree only asymptotically. Fig.~\ref{fig:construction} shows that, even for small \(N\), the finite-sample difference is negligible.

The vertex-growth trends illustrate the effect of different admission policies. While every collision-free sample is inserted into \dense, the admission rate of \ispars decreases with $N$ as the roadmap progressively capture the geometric and inspection structures.
This saturation is not exhibited by \irisc, whose fixed acceptance probability continues to admit a fraction of samples even after newly revealed coverage becomes rare. 
Comparing \ispars with vanilla \spars isolates the cost of the proposed inspection criterion: local coverage preservation produces only a moderate increase in the number of retained vertices, and has a minor effect on the resulting edge count.


The number of collision checks provides a hardware- and implementation- independent proxy for roadmap-construction effort. \spars and \ispars require more checks than \dense because they maintain the paired dense roadmap while additionally evaluating local connections.
The small difference exhibited shows \ispars adds only a little overhead beyond vanilla \spars.
\irisc performs the fewest checks, simply since it utilizes only a fraction of its sampling budget. We note inspection-planning runtime is typically dominated by the downstream solver~\cite{fu2023asymptotically}, while collision checking can be accelerated by modern implementations~\cite{thomason2024motions}.

\section{Conclusion and Future Work}
\label{sec:conclusion}
We introduced \ispars, a task-aware extension of \spars that preserves
dense-roadmap observations through nearby sparse representatives while
retaining the connectivity and asymptotic path-quality guarantees of
\spars.
This structure yields an explicit approximation bound relating optimal
inspection tours on paired dense and sparse roadmaps.
Experiments show that \ispars preserves coverage and high-quality tours
while substantially reducing roadmap size, enabling the GIP solver to
find shorter tours within fixed computation budgets than on the
corresponding dense roadmaps.
To our knowledge, \ispars is the first task-oriented extension of
\spars and the first inspection-roadmap sparsifier with a tour-quality
guarantee.


Several limitations of our work motivate future research.
The presented guarantees are asymptotic and conservative, with stronger finite-time guarantees limited primarily by the connectivity properties inherited from \spars.
Recent progress in deterministic sampling for motion planning may provide a foundation for a sharper finite-time analysis, e.g.,~\cite{Panasoff.Solovey.25}.
%
More broadly, the ideas presented here may extend beyond inspection planning.
Many task-oriented motion-planning problems follow a similar roadmap-to-solver pipeline.
By constructing sparse roadmaps that preserve both motion geometry and task-relevant structure, such pipelines may achieve similar improvements in downstream solver efficiency.

\bibliography{references}

@inproceedings{morgan2026scalable,
  title={Scalable Inspection Planning via Flow-based Mixed Integer Linear Programming},
  author={Morgan, Adir and Solovey, Kiril and Salzman, Oren},
  booktitle={WAFR},
  year={2026}
}

@inproceedings{mizutani2024leveraging,
  title={Leveraging fixed-parameter tractability for robot inspection planning},
  author={Mizutani, Yosuke and Salomao, Daniel Coimbra and Crane, Alex and Bentert, Matthias and Drange, P{\aa}l Gr{\o}n{\aa}s and Reidl, Felix and Kuntz, Alan and Sullivan, Blair D},
  booktitle={WAFR},
  year={2024},
}

@inproceedings{Panasoff.Solovey.25,
  author       = {Itai Panasoff and
                  Kiril Solovey},
  title        = {Effective Sampling for Robot Motion Planning Through the Lens of
 Lattices},
  booktitle      = {RSS},
  year         = {2025},
}

@article{fu2023asymptotically,
  title={Asymptotically optimal inspection planning via efficient near-optimal search on sampled roadmaps},
  author={Fu, Mengyu and Kuntz, Alan and Salzman, Oren and Alterovitz, Ron},
  journal={IJRR},
  volume={42},
  number={4-5},
  pages={150--175},
  year={2023},
  publisher={SAGE Publications Sage UK: London, England}
}

@inproceedings{fu2021computationally,
  title={Computationally-efficient roadmap-based inspection planning via incremental lazy search},
  author={Fu, Mengyu and Salzman, Oren and Alterovitz, Ron},
  booktitle={ICRA},
  pages={7449--7456},
  year={2021},
  organization={IEEE}
}

@article{dobson2014sparse,
  title={Sparse roadmap spanners for asymptotically near-optimal motion planning},
  author={Dobson, Andrew and Bekris, Kostas E},
  journal={IJRR},
  volume={33},
  number={1},
  year={2014},
  publisher={SAGE Publications Sage UK: London, England}
}

@inproceedings{nissoux1999visibility,
  title={Visibility based probabilistic roadmaps},
  author={Nissoux, Carole and Sim{\'e}on, Thierry and Laumond, J-P},
  booktitle={IROS},
  volume={3},
  pages={1316--1321},
  year={1999},
  organization={IEEE}
}

@article{karaman2010incremental,
  title={Sampling-based algorithms for optimal motion planning},
  author={Karaman, Sertac and Frazzoli, Emilio},
  journal={IJRR},
  volume={30},
  number={7},
  pages={846--894},
  year={2011},
  publisher={Sage Publications Sage UK: London, England}
}

@article{cao2025cooperative,
  title={Cooperative aerial robot inspection challenge: A benchmark for heterogeneous multi-{UAV} planning and lessons learned},
  author={Cao, Muqing and Nguyen, Thien-Minh and Yuan, Shenghai and Anastasiou, Andreas and Zacharia, Angelos and Papaioannou, Savvas and Kolios, Panayiotis and Panayiotou, Christos G and Polycarpou, Marios M and Xu, Xinhang and others},
  journal={arXiv preprint 2501.06566},
  year={2025}
}

@article{marble2013asymptotically,
  title={Asymptotically near-optimal planning with probabilistic roadmap spanners},
  author={Marble, James D and Bekris, Kostas E},
  journal={IEEE Transactions on Robotics},
  volume={29},
  number={2},
  pages={432--444},
  year={2013},
  publisher={IEEE}
}

@article{salzman2014sparsification,
  title={Sparsification of motion-planning roadmaps by edge contraction},
  author={Salzman, Oren and Shaharabani, Doron and Agarwal, Pankaj K and Halperin, Dan},
  journal={IJRR},
  volume={33},
  number={14},
  pages={1711--1725},
  year={2014},
  publisher={SAGE Publications Sage UK: London, England}
}

@inproceedings{wang2013fast,
  title={A fast streaming spanner algorithm for incrementally constructing sparse roadmaps},
  author={Wang, Weifu and Balkcom, Devin and Chakrabarti, Amit},
  booktitle={IROS},
  pages={1257--1263},
  year={2013},
  organization={IEEE}
}

@inproceedings{englot2012sampling,
  title={Sampling-based coverage path planning for inspection of complex structures},
  author={Englot, Brendan and Hover, Franz},
  booktitle={ICAPS},
  volume={22},
  pages={29--37},
  year={2012}
}

@article{urtasun2024sparse,
  title={Sparse sampling-based view planning for complex geometries},
  author={Urtasun, Benat and Andonegui, Imanol and Gorostegui-Colinas, Eider},
  journal={IEEE Sensors Journal},
  volume={24},
  number={9},
  pages={14992--15003},
  year={2024},
  publisher={IEEE}
}

@inproceedings{danner2000randomized,
  title={Randomized planning for short inspection paths},
  author={Danner, Tim and Kavraki, Lydia E},
  booktitle={ICRA},
  volume={2},
  pages={971--976},
  year={2000},
  organization={IEEE}
}

@article{kavraki1996probabilistic,
  title={Probabilistic roadmaps for path planning in high-dimensional configuration spaces},
  author={Kavraki, Lydia E and Svestka, Petr and Latombe, J-C and Overmars, Mark H},
  journal={Tran.\ on Robotics and Automation},
  volume={12},
  number={4},
  year={1996},
  publisher={IEEE}
}

@article{peleg1989graph,
  title={Graph spanners},
  author={Peleg, David and Sch{\"a}ffer, Alejandro A},
  journal={Journal of graph theory},
  volume={13},
  number={1},
  pages={99--116},
  year={1989},
  publisher={Wiley Online Library}
}

@inproceedings{thomason2024motions,
  title={Motions in microseconds via vectorized sampling-based planning},
  author={Thomason, Wil and Kingston, Zachary and Kavraki, Lydia E},
  booktitle={ICRA},
  year={2024},
  organization={IEEE}
}

\end{document}